\documentclass{article} 
\usepackage{times,fullpage}
\usepackage{graphicx,authblk}
\usepackage{algorithm,appendix}
\usepackage{algorithmic}
\usepackage{ natbib}
\usepackage{url}
\usepackage{verbatim}

\def\cRn{\mathcal{R}_n}
\def\cFg{\mathcal{F}_g}
\def\chatRn{\hat{\mathcal{R}}_n}
\def\cN{\mathcal{N}}
\def\cZ{\mathcal{Z}}
\def\Uhat{\hat{U}}
\def\Ahat{\hat{A}}
\def\Uhatomega{\hat{U}_{\Omega}}
\def\Uomega{U_{\Omega}}
\def\what{\hat{w}} 
\def\Lhat{\hat{L}}

\def\fhat{\hat{f}}
\def\Omegax{\Omega_x}
\def\xomega{x_\Omega}
\def\fest{\hat{f}_{\hat{w},\hat{U},\gamma}}
\def\ones{\mathbbm{1}}

\usepackage{amsfonts}
\usepackage{url}
\usepackage{bbm}
\usepackage{multirow}
\usepackage{algorithm,algorithmic,amsmath,amsthm,amsfonts}
\usepackage{verbatim,graphicx,subfigure}
\usepackage{mathtools}
\usepackage[usenames,dvipsnames,svgnames,table]{xcolor}
\newcommand{\BlackBox}{\rule{1.5ex}{1.5ex}}  
 
\newtheorem*{assumption*}{Assumption}
 
\newtheorem{thm}{Theorem}
\newtheorem{proof sketch}{Proof Sketch}
\newtheorem{lemma}{Lemma}

\newtheorem{defn}{Definition}

\def\bbE{\mathbb{E}}

\def\bbR{\mathbb{R}}

\def\cR{\mathcal{R}}

\def\cF{\mathcal{F}}

\def\ones{\mathbbm{1}}

\def\diag{\diag}

\def\cR{\mathcal{R}}

\newcommand{\defeq}{\mbox{$\;\stackrel{\mbox{\tiny\rm def}}{=}\;$}}

\DeclareMathOperator\erf{erf}
\DeclarePairedDelimiter\abs{\lvert}{\rvert}%
\DeclarePairedDelimiter\norm{\lVert}{\rVert}%

\makeatletter
\let\oldabs\abs
\def\abs{\@ifstar{\oldabs}{\oldabs*}}
\let\oldnorm\norm
\def\norm{\@ifstar{\oldnorm}{\oldnorm*}}
\makeatother

\begin{document}
\title{Sparse Linear Regression With Missing Data}

\author[1]{Ravi Ganti \thanks{gantimahapat@wisc.edu}}
\author[2]{Rebecca M. Willett\thanks{rmwillett@wisc.edu}}
\affil[1]{Wisconsin Institutes for Discovery, 330 N Orchard St, Madison, WI, 53715}
\affil[2]{Department of Electrical and Computer Engineering, University of Wisconsin-Madison, Madison, WI, 53706}
\renewcommand\Authands{ and }
\maketitle

\begin{abstract} This paper proposes a fast and accurate method for sparse regression in the presence of missing data. The underlying statistical model encapsulates the low-dimensional structure of the incomplete data matrix and the sparsity of the regression coefficients, and the proposed algorithm jointly learns the low-dimensional structure of the data and a linear regressor with sparse coefficients. The proposed stochastic optimization method, Sparse Linear Regression with Missing Data (SLRM), performs an alternating minimization procedure and scales well with the problem size. Large deviation inequalities shed light on the impact of the various problem-dependent parameters on the expected squared loss of the learned regressor. Extensive simulations on both synthetic and real datasets show that SLRM performs better than competing algorithms in a variety of contexts.  \end{abstract}

\section{Introduction}
Modern statistical data analysis requires tools that can handle complex, large scale datasets. Due to constraints in the data collection process, one often has incomplete datasets, i.e., datasets with missing entries, with which we need to perform statistical inference. For instance, in sensor networks, readings from all the sensors might not be available at all the times because of malfunctions in sensors, or simply because it is too expensive to gather readings from all the sensors at all the times.  Similarly, when conducting surveys, responders may avoid answering certain questions for the sake of privacy or otherwise, leading to missing entries in survey data. Recommender systems, implement algorithms that are required to train on data with missing entries. For example, popular recommendation engines such as Netflix, online radio services such as Pandora, social networks such as Facebook, LinkedIn regularly deal with prediction problems involving data with missing entries. An ever increasing demand to gather as much data as possible, clean or not,  in this big-data era, has led to the need for statistical methods that can deal with not just clean data but also noisy data with missing components. 

The focus of this paper is on sparse linear regression when the feature vectors or design matrix have missing elements. Matrix completion methods allow missing elements to be imputed accurately, but generally do not account for any auxiliary label information. Similarly, sparse linear regression and LASSO methods rely upon a fully-known design matrix. One might imagine using matrix completion to impute missing entries and then applying sparse linear regression methods to the completed design matrix; {\em we demonstrate that this two-stage approach is sub-optimal, and propose a unified regression framework that yields significantly better performance in a variety of tasks.}


\subsection{Contributions.} 
Our contributions are as follows
\begin{enumerate}
\item In this paper, we propose a statistical model (Section~\ref{sec:model}) for the problem of sparse linear regression with missing data. Our model captures low-rank structure in the data and sparsity of the regression coefficients in the lower dimensional representation of the data. 
\item We provide an optimization-based approach that simultaneously learns the underlying subspace structure and the sparse regression coefficients (Section~\ref{sec:learning_alg}). Our optimization algorithm, called SLRM, takes a combination of stochastic first order and second order steps, alternating between the different parameters of the proposed statistical models. 
\item We establish large deviation bounds (Section~\ref{sec:generr}) for the risk of the regressor learned by our algorithm in terms of the empirical loss, the ambient dimension $D$, and a parameter $\gamma$ used by our learning algorithm.  Using our performance bounds we can understand the impact of the amount of missingness on the training error and the test error.
\item We provide extensive experimental results (Section~\ref{sec:expts}) on synthetic and real datasets, comparing the performance of SLRM and a competing algorithm. From our experimental results, we conclude that SLRM has good noise tolerance properties, and uses the label information well to learn a good regressor, as measured by its mean squared error on a test dataset with missing features. 
\end{enumerate}
\section{Problem Formulation: Sparse Regression With Missing Data} 
\label{sec:model}
Given $D$-dimensional labeled data with missing features, we are interested in prediction, particularly regression problems.  Let  $X=(x_1,x_2,\ldots,x_n)\in\bbR^{D\times n}$ be a data matrix, where the columns have been sampled i.i.d.\ from a distribution. Since we are interested in regression problems with missing data, we do not get to see all the entries of the data matrix $X$. To formalize this notion, let $\Omega_1,\ldots,\Omega_n$ be subsets of $\{1,2,\ldots,D\}$. Given an index set $\Omega$, let $P_{\Omega}(x)$ denote a sub-vector of $x$ consisting of elements whose indices are elements of the set. 
 We observe a dataset $(P_{\Omega_1}(x_1),y_1,\Omega_1),\ldots,(P_{\Omega_n}(x_n),y_n,\Omega_n)$ of size $n$, i.e., we observe only a few entries of the data points $x_1,\ldots,x_n$, where the entries are indexed by the sets $\Omega_1,\ldots,\Omega_n$ respectively. We call the vector $Y=(y_1,y_2,\ldots,y_n)^\top$ the label vector. Given this training data, we are required to learn a regressor, which when given an unseen test point $(P_\Omega(x),\Omega)$, predicts a label $\hat{y}$ that is close to the true label of $x$. In order to solve this problem, we consider the following statistical model:
\begin{align}
  X&=U_*A_*+\epsilon_{X}\label{eqn:mod1}\\
  Y&= A_*^\top w_*+\epsilon_{Y}\label{eqn:mod2},
\end{align}
where $w_*$ is a sparse vector in $\bbR^d$, $U_*$ in $\bbR^{D\times d}$ ($d<D$) is a matrix with full column rank, and $A_{*}=[\alpha_{1*},\ldots,\alpha_{n*}]$ is a matrix in $\bbR^{d\times n}$. We call $\alpha_{i*}$  the code of $x_i$ w.r.t.\ the 
matrix $U$. The vector $\epsilon_Y=(\epsilon_{y_1},\ldots,\epsilon_{y_n})^\top$ is random noise that is independent of other problem parameters such as $U,A,w,\Omega_1,\ldots,\Omega_n$. Similarly $\epsilon_X=[\epsilon_{x_1},\ldots,\epsilon_{x_n}]$ is a noise matrix with i.i.d.\ entries, sampled independently of other problem parameters. 

Our statistical model given in Equations~\ref{eqn:mod1},\ref{eqn:mod2} is motivated by the fact, for many data matrices of interest, even though the ambient data dimensionality is large, the data lies close to a lower dimensional subspace of dimensionality $d$. Given, this $d$-dimensional representation of the data, we are interested in learning a linear regressor with sparse coefficients that predicts the labels well. 

To the best of our knowledge, for the problem of regression with missing data, our work is the first work that simultaneously exploits both a low-rank structure of the incomplete data matrix and the sparsity of regressor. The assumption of a parametric model for our regression problem allows us to go beyond the transductive setting which was inherent in the approach of~\cite{goldberg2010transduction} (as detailed in Section~\ref{related}).
While we consider $d<D$, we are also interested in cases where $d$ is of the same order as $D$ and the regressor is sparse in the lower-dimensional representation of the data. This model is relevant to many applications, as described in Section~\ref{sec:expts}

For instance, in a sensor network $D$ sensors listen to $d$ sources. 
As one would expect, this sensor data is far from being ``clean'': it is usually noisy, and has missing entries.
A common approach in analyzing such sensor data is to perform a subspace analysis of the sensor data~\citep{tuncer2009classical,krim1995sensor,roy1989esprit} and find the best fit $d$-dimensional subspace of the data. 
For modern sensor networks, both $D$ and $d$ are large; that is, a large number of heterogeneous sensors listen to a large number of sources. Exploiting the underlying $d$-dimensional structure during regression yields increased robustness to noise and missing data. 

\textbf{Notation}. Like in the definition of $A_*$, $A=[\alpha_1,\alpha_2,\ldots,\alpha_n]$, $\hat{A}=[\hat{\alpha}_1,\hat{\alpha}_2,\ldots,\hat{\alpha}_n]$. 
Given a matrix $M$, denote $P_{\Omega}(M)$ as the matrix whose rows are those rows of $M$ whose indices are elements of the set $\Omega$. For example, if $\Omega=\{1,3,4\}$, then $P_{\Omega}(M)$ has rows 1,3,4 of matrix $M$. At times, for ease of notation we may write $x_{\Omega},M_{\Omega}$ to denote $P_{\Omega}(x), P_{\Omega}(M)$ respectively. By $I_d$ we represent an identity matrix with $d$ rows.

\section{Related Work}
\label{related}
Our statistical model bears resemblance to the statistical model used in partial least squares (PLS)~\citep{hastie2003esl}. However, unlike PLS we enforce additional sparsity assumptions and can handle missing data. Dictionary learning was introduced for unsupervised data analysis for better data representation~\citep{maurer2010k,vainsencher2011sample}. The idea is to learn a dictionary so that each data point could be represented well as a sparse linear combination of the columns of the dictionary. Dictionary learning has also been extended to prediction problems~\citep{mairal2012task,szlam2009discriminative}, where the problem is to learn a dictionary for the prediction problem at hand. The problem that we tackle in this paper can be seen as learning a dictionary for prediction problems in the presence of missing data. Sufficient dimensionality reduction (SDR)~\citep{suzuki2013sufficient,fukumizu2009kernel}, is a form of supervised dimensionality reduction, where the problem is to find a central subspace $Z$ such that the prediction task is independent of the unlabeled data given the projection $\Pi_Z X$ of unlabeled data onto the central subspace. SDR focuses on achieving conditional independence between $Y$ and $X$ given $\Pi_Z X$ without making any assumptions on the functional dependency of the prediction task on the central subspace. SDR does not fully exploit linear relationships between labels and features that arise in many practical settings, and the problem of SDR with missing data has not been investigated. ~\citet{loh2011high} investigate non-convex algorithms based on maximum likelihood estimation for the problem of high-dimensional regression with missing data. However, they work with a different statistical model which does not capture the low-rank structure of the data and assumes that the regressor is sparse in the ambient space. In contrast, our statistical model explicitly assumes that the missing data matrix has a low-rank structure and exploits this low-rank structure in data to learn a regressor with sparse coefficients in the low-dimensional representation of the data. Another closely related work is that of~\citep{goldberg2010transduction}, where the authors consider the problem of multi-task regression with missing data features and missing labels. The authors pose this problem as a matrix completion problem of the matrix formed by the concatenation of the data and label matrices. However, the authors deal with the transductive setting only -- their approach does not allow one to predict a label for a new test datapoint. 
In contrast, this paper exploits an alternative statistical model for how the labels are generated that allows prediction on new test datapoints.
Finally, Principal Component Regression (PCR)~\citep{hastie2003esl} is a dimensionality-reduction based procedure for regression without missing data. PCR first performs PCA on the unlabeled dataset, followed by least squares regression in the PCA space. This two-step approach does not exploit label information when estimating the underlying low-dimensional model; the limitations of this choice are detailed in Section~\ref{sec:learning_alg}. 

As mentioned above, we use stochastic optimization methods that can operate on streaming data to ensure scalable algorithms. Thus the low-rank structure in our problem is estimated using techniques drawn from the subspace tracking literature. Oja's method~\citep{oja1982simplified}, PAST~\citep{yang1995projection} and variations such as OPAST~\citep{abed2000fast} perform subspace tracking when there is no missing data. More recent developments, such as GROUSE~\citep{balzano2010online} and PETRELS~\citep{chi2012petrels} can handle missing data quickly and accurately. However, these algorithms are inherently unsupervised and hence do not directly address the supervised regression problem considered in this paper.

\section{An Optimization Approach And A Learning Algorithm} \label{sec:learning_alg}
 Before we describe our optimization based approach to the problem considered in this paper, we discuss a multi-step approach (essentially an extension of PCR to missing data problems) that exploits both the low rank of the incomplete data matrix  and the sparsity of the regression coefficients:
\begin{enumerate}
\item Solve the following optimization problem
\begin{equation}\min_{U,\alpha_1,\ldots, \alpha_n} \sum_{i=1}^n||P_{\Omega_i}(x_i)-P_{\Omega_i} (U\alpha_i)||_2^2.
\label{pcr}
\end{equation}
The above problem aims to consider a decomposition of the incomplete data matrix $X$ as the product of two matrices $U, A$, such that the Frobenius norm of the
difference between $X$ and $UA$ over the observed entries is minimized. This problem has been studied in the matrix completion literature~\citep{koren2009matrix, jain2013low}, and in the subspace identification and tracking literature~\citep{chi2012petrels,hua1999new}. A standard approach to solving this problem is via alternating minimization, where we alternate between optimization w.r.t. $U$ and the vectors $\alpha_1,\ldots,\alpha_n$.  In the special case that for all $i=1,\ldots,n$, $\Omega_i=\{1,2,\ldots,D\}$ (as in classical PCR), the solution to the above problem is obtained by performing PCA of the data matrix. 

\item Let  $\hat{U}$, $\hat{A}\defeq [\hat{\alpha}_1,\ldots,\hat{\alpha}_{n}]$ be the solution of \eqref{pcr}. Learn a linear regressor with sparse coefficient, using $\hat{A}$ as the design matrix and by solving the following $\ell_1$ penalized problem
\begin{equation}\hat{w}=\arg\min_{w}\frac{1}{n}||Y-\hat{A}^\top w||_2^2+\lambda ||w||_1.
\label{lasso}
\end{equation}
\end{enumerate}
 We call the above two step procedure  MPCR\footnote{M in MPCR stands for missing}. Note that in Step 1 of MPCR, the label data is not used. A merit of MPCR over other approaches previously proposed for our problem is that MPCR explicitly utilizes the low-rank structure of the data, and a linear model for the regression task at hand. 

However, such multi-step algorithms that do not utilize the label information in all the steps are  inherently label-inefficient. First, such multi-step algorithms fail to exploit information about $U_{*}$ reflected by the labels. Second, the estimate $\hat{U}$ is one basis (of many potential bases) of the underlying subspace. Since we perform sparse regression on the subspace coefficients, the choice of basis matters. However, without label information, we have no way of knowing which basis rotation is best. Third, MPCR solves a harder problem than necessary. To see why, note that when $w_{*}$ is sparse, then for the purpose of prediction, only those rows of $A_{*}$ and columns of $U_*$ that correspond
to the non-zero coordinates of $w_{*}$ matter. 

In general, any multi-step procedure that does not utilize label information when estimating the underlying subspace will be label-inefficient for learning a good predictor. This is particularly true when both $D$ and $d$ are of the same order, as in the sensor network problems described in Section~\ref{sec:model}.  This is because when $d$ is comparable to $D$, there is a good deal of information in the labels that can be used to efficiently estimate the underlying subspace. We observe this in our experiments too, where on the CT slice dataset, where $D=384,d=181$, MPCR gives substantially worse performance than our proposed algorithm.

Armed with these insights, we are interested in procedures that utilize label information fully. We do this by proposing a joint optimization procedure that  
simultaneously learns all the relevant variables in our model. 
\subsection{Learning Via Joint Optimization}
Given constants $\lambda_1, \lambda_2,\lambda_3>0$, we propose to solve the following optimization problem.
\begin{equation}
\begin{aligned}
\label{eqn:mainopt}
  \underset{U,A,w} {\text{minimize}}&~
  \frac{\lambda_1}{n}\sum_{i=1}^n||P_{\Omega_i}(x_i)-P_{\Omega_i}(U\alpha_i)||_2^2+
  &~\frac{1}{n}
  ||Y-A^\top w||_2^2+\lambda_2||w||_1 +\lambda_3||w||_2^2\\
  \text{subject to} &~~U^\top U=I_d,
\end{aligned}
\end{equation}
where $U\in \mathbb{R}^{D\times d},A=[\alpha_1,\ldots,\alpha_n]\in \mathbb{R}^{d\times n},w\in
\mathbb{R}^d$. Like in MPCR, the first term corresponds to a matrix
completion term. The second term in the
above optimization formulation measures the squared loss of a regressor
$w$ on a low-dimensional representation of the training data.  The third
term in our optimization formulation is the $\ell_1$ norm penalty which encourages sparse $w$. Finally the last term is motivated by elastic net type formulation for sparse prediction. 
We optimize over $U,w,A$, under the constraints that the columns of $U$ be orthonormal to each other to ensure uniqueness of the solution. 
The above optimization procedure outputs $\what,\Uhat,\Ahat$. Given an unlabeled data point with missing entries, $(P_\Omega(x),\Omega)$, and a constant $\gamma\in [0,1)$, we first project the point onto the subspace spanned by the columns of the matrix $\Uhat_{\Omega}$, to obtain 
$\tilde{x}=(\Uhatomega^\top\Uhatomega)^{-1}\Uhatomega^\top\xomega$. Our regressor, $\fhat_{\what,\Uhat,\gamma}$ then predicts the label of $(P_{\Omega}(x),\Omega)$ as 
\begin{equation}
  \label{eqn:freg}
  \fhat_{\what,\Uhat,\gamma}(P_{\Omega}(x),\Omega)=
    \what^\top\tilde{x} \ones_{\{||(\Uhatomega^\top\Uhatomega)^{-1}||_2 \leq \frac{D}{m(1-\gamma)}\}},
\end{equation}
where $\ones_{\{\cdot\}}$ is the indicator function.
Whenever $||(\Uhatomega^\top\Uhatomega)^{-1}||$ is large, it implies that the missing entries will not allow accurate subspace projection; in this case, our  method outputs $0$. 
A good choice of $\gamma$ depends on $d$ and $|\Omega|$, and we shall discuss this in detail in Section~\ref{sec:generr}.

\subsection{Solving The Optimization Problem}
The optimization problem shown in ~\eqref{eqn:mainopt} is individually convex in the optimization variables $U,A,w$, but jointly non-convex. We solve this problem via an alternating minimization approach, where we minimize over $U,A,w$ alternatively. In addition, we adopt a stochastic optimization approach. Our algorithm is called Sparse Linear Regression with Missing data (SLRM). SLRM makes a pass over the dataset, and each time uses a single data point $(P_{\Omega_t}(x_t),y_t,\Omega_t)$ to make updates to all the parameters. SLRM uses stochastic second order steps to update matrices $U,A$, and stochastic first order steps to update $w$ vector. Algorithm~\ref{alg:slrm} provides a pseudocode of our
proposed stochastic optimization algorithm. There are six main steps,
which we shall discuss below in detail. 

\textbf{Initialization.} In Step 1 we initialize $U, w$ to $\hat{U}_0,\hat{w}_0$. $\hat{U}_0$ is obtained by performing SVD of the incomplete data matrix with 0's filled in the missing entries. The left singular vectors corresponding to the top $d$ singular values form the $\hat{U}_0$ matrix. Similar initialization techniques have been proposed in matrix completion literature~\citep{jain2013low,hardt2013understanding,koren2009matrix}. We initialize $\hat{A}_0$ by projecting each $P_{\Omega_i}(x_i)$ onto the subspace spanned by $\hat{U}_0$. We initialize $\hat{w}_0$ by solving the LASSO regression on $Y$ and $\hat{A}_0$, similar to \eqref{lasso}.
 We initialize matrices, $R^{0}_1,R_2^{0},\ldots,R_D^{0}$ to a multiple of the identity matrix. These matrices are required in Step 6 of our algorithm.

\textbf{Updating $A$}. In round $t$, SLRM uses $(P_{\Omega_t}(x_t),y_t,\Omega_t)$ to update our estimate of the $A$ matrix. Since $(P_{\Omega_t}(x_t),y_t,\Omega_t)$ is only responsible for the $t^{\text{th}}$ column of matrix $A$, in Step 5 of SLRM we replace the $t^{\text{th}}$ column of $\hat{A}_{t-1}$, with $\hat{\alpha}_t$, to obtain $\hat{A}_t$. This update reduces to a simple unconstrained quadratic optimization problem over $\alpha_t$, which can be solved in closed form by solving a system of linear equations.
\newline \newline
\textbf{Updating $U$}. In Step 6, we update $\hat{U}_{t-1}$ by using
the MODIFIED-PETRELS (MP) routine. The MP routine is inspired by  the
PETRELS algorithm~\citep{chi2012petrels}, which was designed for
estimating subspaces from streaming data with missing entries. PETRELS can be seen as solving the optimization problem
$\underset{U}\arg\min\sum_{i=1}^n f_i(U)$, where
$f_i(U)=||P_{\Omega_i}(x_i-U\alpha_i)||^2$, and $\alpha$'s correspond to a projection of the observations onto the current subspace estimate. MP solves the same optimization problem, but with the $\alpha_i$ from Step 5, {\em which uses label information}. 
Both methods update $\hat{U}_{t-1}$ to $\tilde{U}_{t}$ by performing a single stochastic Newton step on $f_t(U)$, starting at $\hat{U}_{t-1}$, and using $P_{\Omega_t}(x_t),\Omega_t,\alpha_t$. This Newton step can be implemented efficiently using recursive least squares, and a pseudocode for the MP routine is available in Algorithm 2.

\textbf{Orthonormalization of updated $U$.} Since we are optimizing
over the manifold of rectangular matrices with orthonormal columns, we
perform an orthonormalization step in Step 9, by solving the following nearest orthogonal matrix problem: $U_t=\underset{U} \arg\min||U-\tilde{U}_{t}||_F$ subject to $U^\top U=I_d$. This problem has the closed form solution as shown in Step 7 of SLRM. Note that by construction, our orthonormalization step always guarantees, that
the columns of $\hat{U}_{t}$ always span a $d$-dimensional subspace of $\mathbb{R}^D$.
\newline \newline
\textbf{Updating $w$}. In Step 8, we perform one step of the stochastic projected gradient algorithm w.r.t.\ $w$. Our objective function is $\underbrace{\frac{1}{n}||Y-\hat{A}_t^\top w||_2^2+\lambda_3||w||_2^2}_{F(w)}+\lambda_2||w||_1$ . A step of the stochastic projected gradient method requires us to calculate a noisy estimate of, $\nabla F(w_{t-1})$, using $\alpha_t,y_t$, followed by an application of the prox operator corresponding to $||w||_1$. 
\newline\newline  
\textbf{Validation Steps.} We let $\hat{w}$ and $\hat{U}$ denote the estimate stored at the end of the previous round. In Steps 9-12, we determine, using a hold-out validation set, whether $\hat{w}$ and $\hat{U}$ form a better regressor than $\hat{w}_t$ and $\hat{U}_t$. The pair that achieves smaller hold-out error is then stored as $\hat{w}$ and $\hat{U}$ for the next round.  

These steps are required since we are solving a  non-convex optimization problem, and hence it is not necessarily true that $(\hat{w}_T,\hat{U}_T)$ leads to the best regressor. 

Note that SLRM can easily be modified to handle the case where we have semi-supervised data. If we get unlabeled data in a round $t$, then we perform the optimization problem in Step 5 of the SLRM algorithm without the term $(y_t-\hat{w}_{t-1}^\top\alpha)^2$, and simply skip the weight update in Step 8.

\subsection{Computational Complexity and Convergence}  Step 5 of SLRM solves a system of linear equations and takes $O(|\Omega_t|d^2)$ time. Step 6 of SLRM allows a parallel implementation, where the rows of the matrix $\tilde{U}_{t}$ are updated in parallel. This takes $O(|\Omega_t| d^2)$ time. Finally, Step 7 of SLRM is the classical orthogonal Procrustes problem and takes $O(Dd^2+d^3)$ time. Hence, all together the time complexity of our algorithm is $O(|\Omega_t|d^2+Dd^2)$. In particular, since steps 5,7 are well studied numerical problems, they have efficient numerical implementations available. Since, our algorithm is built on exploiting the low rank structure of the missing data matrix, we attempt to get a rough estimate of the subspace spanned by the data.  Algorithms that attempt to estimate the subspace spanned by missing data such as GROUSE~\citep{balzano2010online}, PAST~\citep{yang1995projection} need to expend $O(|\Omega_t|d^2)$ computation. Hence, it appears at least $O(|\Omega_t|d^2)$ amount of computation is inevitable. The overall higher computational complexity of our SLRM over algorithms such as PAST, GROUSE etc. is because of the additional prediction task that we aim to solve with SLRM.

SLRM, like other task driven dictionary learning approaches~\citep{mairal2012task}, uses a combination of stochastic updates and alternating minimization for a non-convex objective function. Empirically we observe similar convergence behavior to that reported in~\citep{mairal2012task}; to the best of our knowledge, no formal convergence guarantees are available for SGD based approaches to the task driven dictionary learning problem~\citep{mairal2012task}. Note that~\citet{mairal2009online} also uses stochastic updates and alternating minimization for a biconvex objective and describes associated convergence analysis; however, the problem considered in that paper is far simpler than ours, in that it was not task-driven and didn't handle missing data.
\newline
\begin{algorithm*}
  \label{alg:slrm}
  \caption{SLRM. Input: Parameters
    $\lambda_1,\lambda_2,\lambda_3, \delta >
    0, 0\leq \gamma<1$, Output: $\what,\Uhat$}
  \begin{algorithmic}[1]
    \STATE Initialize $\what=\hat{w}_0, \Uhat=\hat{U}_0,\hat{A}=\hat{A}_0$, $(R_1^0)^{\dagger}=\delta I_d, (R_2^0)^{\dagger}=\delta I_d,\ldots,  (R_D^0)^{\dagger}=\delta I_d$.
    \STATE Initialize $curr\_best\_val\_err=\infty$.
    \FOR {$t=1,2,\ldots n$}
    \STATE Receive $(P_{\Omega_t}(x_t),y_t,\Omega_t)$
    \STATE Replace $t^{\text{th}}$ column of $\hat{A}_{t-1}$ with $\hat{\alpha}_t$ to get $\hat{A}_t$. $\hat{\alpha}_t$ is given by, 
    \begin{align*}
      \hat{\alpha}_t&=\arg\min_{\alpha} \lambda_1
      ||P_{\Omega_t}(x_t)-P_{\Omega_t}(\hat{U}_{t-1}\alpha)||_2^2+
      (y_t-\hat{w}_{t-1}^\top\alpha)^2
    \end{align*}
    \STATE Update $U$, using the current sample $(P_{\Omega_t}(x_t),\Omega_t)$, as follows
    \begin{equation*}
      \tilde{U}_t,R_{1}^t,R_2^t,\ldots,R_D^t\leftarrow \text{Modified-PETRELS}(\hat{U}_{t-1},P_{\Omega_t}(x_t),\Omega_t,\hat{\alpha}_{t},(R_{1}^{t-1})^{\dagger},\ldots,(R_{D}^{t-1})^{\dagger})
    \end{equation*}
    \STATE Orthonormalize by, $\hat{U}_t\leftarrow \tilde{U}_t (\tilde{U}_t^\top\tilde{U}_t)^{-1/2}$
    \STATE Perform stochastic proximal gradient type update using the following equations
    \begin{align*} \hat{w}_{t}&=\text{prox}_{\eta_t\lambda_2,||\cdot||_1}\left[\hat{w}_{t-1}-\eta_t\Bigl(2(\hat{\alpha}_t\hat{\alpha}_t^\top\hat{w}_{t-1}-y_t\hat{\alpha}_t)+\lambda_3\hat{w}_{t-1}\Bigr)\right]
\end{align*}
\STATE $\text{val\_err}=$Validation-Error($\hat{w}_t,\hat{U}_t,\gamma$)
\IF{$val\_err<curr\_best\_val\_err$}
\STATE $\text{curr\_best\_val\_err}=\text{val\_err}$.
\STATE $\what\leftarrow \hat{w}_t,\Uhat\leftarrow \hat{U}_t$
\ENDIF
\ENDFOR
\end{algorithmic}
\end{algorithm*}
\section{Generalization Error Bounds}
\label{sec:generr}
\begin{defn}
  Given a $\gamma\in [0,1)$, let 
  \begin{equation}
  \label{eqn:model_regressor}
  f_{w,U}(x_{\Omega},\Omega)=
    w^\top(\Uomega^\top\Uomega)^{-1}\Uomega^\top\xomega \ones_{\{||(\Uomega^\top\Uomega)^{-1}||_2 \leq \frac{D}{m(1-\gamma)}\}},
\end{equation}
and
  \begin{equation*}
  \begin{split}
  \cF_{\gamma}\defeq\{f_{w,U}| w\in \bbR^d, U\in \bbR^{D\times d}, ||w||_1\leq R_1, U^\top U=I_d, 
  f_{w,U} ~\text{is as given in Equation} ~\ref{eqn:model_regressor}\}.
  \end{split}
\end{equation*}
\end{defn}
Our main theorem is as follows
\begin{thm}
  \label{thm:main}
Consider a regression problem where a training set of $n$ data samples $(P_{\Omega_i}(x_i),y_i,\Omega_i)$ are
sampled i.i.d. from a probability distribution, with $|y_i|\leq B_Y$,
$||x_i||_{\infty}\leq B_X$, almost surely. Let each $\Omega_i$ be a set of cardinality
$m$, chosen uniformly at random with
replacement from the set $\{1,2,\ldots,D\}$. Let $b\defeq
2(B_Y+\frac{DR_1}{m(1-\gamma)})^2$. Choose a $\gamma\in [0,1)$. Let $\Lhat(f)\defeq\frac{1}{n}\sum_{i=1}^n (y_i-f(P_{\Omega_i}(x_i),\Omega_i))^2, L(f)\defeq\bbE_{x,y,\Omega} (y-f(P_{\Omega}(x),\Omega))^2$
Then for any $\delta >0$, and a universal constant $K>0$, we have with probability at least $1-\delta$, over a random sample of size $n$, for all $f\in \cF_{\gamma}$,
\begin{equation}
   L(f)\leq \Lhat(f) +K\Biggl[\sqrt{\Lhat(f)}\Biggl( \left(\frac{m}{n}+\frac{1}{\sqrt{n}}\right)\frac{DR_1B_X}{1-\gamma}
   +\sqrt{\frac{b\log(1/\delta)}{n}}\Biggr)+\frac{b\log(1/\delta)}{n}
   + \log^3(n)\left(\frac{m}{n}+\frac{1}{\sqrt{n}}\right)^2\left(\frac{DR_1B_X}{1-\gamma}\right)^2\Biggr]
\end{equation} 
\end{thm}
For appropriate values of $\lambda_1,\lambda_2,\lambda_3,R_1$,  the output of SLRM $\fhat_{\what,\Uhat,\gamma} \in \cF_{\gamma}$. 
The complete proof is in the appendix. Here, we shall provide a brief synopsis of the proof.
\begin{algorithm}
  \label{alg:mod_petrels}
  \caption{MODIFIED-PETRELS. Input: $\hat{U}_{t-1},P_{\Omega_t}(x_t), \Omega_t,\hat{\alpha}_t,R_1^{t-1},\ldots,R_D^{t-1}$. Output: $\tilde{U}_{t},R_1^{t},\ldots,R_D^t$}
  \begin{algorithmic}[1]
    \FOR{$j=1,\ldots, D$}
    \STATE $\beta_j^t=1+\hat{\alpha}_t^\top(R^{t-1}_j)^{\dagger}\hat{\alpha}_t$
    \STATE $v_j^t=(R_j^{t-1})^{\dagger}\hat{\alpha}_t$
    \STATE $p_{j}^t=\mathbbm{1}[j\in\Omega_t]$
    \STATE $(R^{t}_{j})^{\dagger}=(R^{t-1}_{j})^{\dagger}-p_{j}^t(\beta^{t}_{j})^{-1}v^{t}_{j}(v^{t}_{j})^\top$
    \STATE $\tilde{U}_{t,j}=\hat{U}_{t-1,j}+p_{j}^{t}(x_{t,j}-\hat{\alpha}_t^\top\hat{U}_{t-1,j})(R^{t}_j)^{\dagger}\alpha_t$
    \ENDFOR
  \end{algorithmic}
\end{algorithm}
\begin{proof sketch}
Our proof uses standard large deviation results connecting $L(f)$ and $\Lhat(f)$, similar to \citep[Thm. 1]{srebro2010smoothness}). We upper bound the Rademacher complexity of the function class $\cF_{\gamma}$; Lemmas 2 and 3 in the appendix show how to perform these calculations. 
\end{proof sketch}
We would like to remark that in Theorem~\ref{thm:main} we assumed that $|\Omega_i|=m$ for all $i$. This assumption is only a technical convenience and allows us to state our result in the cleanest possible way. In general, we wish to choose $\gamma$ so that both (a) the empirical error $\hat{L}(\fest)$ is small and (b) the R.H.S. of the inequality in Theorem 1 (which scales like $(1-\gamma)^{-2}$) is small.  A similar trade-off can also be found in structural risk minimization, commonly studied in classical supervised learning, where we know that functions belonging to a richer class have smaller training error, but potentially larger upper bounds on their generalization error.

Specifically, the R.H.S. of the inequality in Theorem~\ref{thm:main} depends on a term of the form  $(\frac{DR_1M}{1-\gamma})^2$. This term can be roughly thought of as a measure of the complexity of the function class, $\cF_{\gamma}$. A large $\gamma$ would imply that we are learning from a richer class of functions and, as can be seen from Theorem~\ref{thm:main}, the upper bound on the risk of functions in the class $\cF_{\gamma}$ will be potentially larger. Thus we wish to keep $\gamma$ as small as possible. 

On the surface, it may appear that $\gamma$ must be close to one to yield a small empirical error (by not predicting zero values). However, in many settings, it is possible to choose a small value of $\gamma$ and still have a low empirical error. To see this, note that 
 larger $m \defeq |\Omega_i|$ leads to easier learning problems, and hence smaller error rates. Let $\hat{S}$ be the subspace spanned by the columns of
$\Uhat$. Let $\mu(\hat{S})\defeq\frac{D}{d}\max_j ||P_{\hat{S}} e_j||^2$, where
$P_{\hat{S}}$ is the projection operator onto $\hat{S}$, and $e_j$ is the standard basis element in $D$ dimensions. $\mu(\hat{S})$ is known as coherence of subspace $\hat{S}$~\citep{candes2009exact}. It is well known~\citep{balzano2010high} that
$||(\Uhatomega^\top\Uhatomega)^{-1}||\leq \frac{D}{m(1-\gamma_1)}$, with
probability at least $1-\delta$ over the random choice of $\Omega$,
where
$\gamma_1=\sqrt{\frac{8d\mu(\hat{S})\log(\frac{2d}{\delta})}{3m}}$. 
Hence, it is enough to set $\gamma\geq \sqrt{\frac{8d\mu(\hat{S})\log(\frac{2d}{\delta})}{3m}}$
From our previous discussions, we know that a small $\gamma$ would mean that the complexity of $\cF_{\gamma}$ is also small. To see how this affects
$\Lhat(\fhat_{\what,\Uhat,\gamma})$, notice that because with probability at least $1-\delta$, $||(\Uhatomega^\top\Uhatomega)^{-1}||\leq \frac{D}{m(1-\gamma)}$, we can claim that  on an expectation we are guaranteed
to make a non-zero prediction on less than a fraction $\delta$ of our
training examples. Hence by choosing a large $m$, we are guaranteed
that we can work on a sufficiently small function class $\cF_\gamma$,
and yet not incur a large training error.  This implies from Theorem~\ref{thm:main}, that $L(\fhat_{\what,\Uhat,\gamma})$ is small. Hence, the correct choice of $\gamma$ depends on $m$, and a suitable choice of $m$ depends on $d$. We now have a nice
interplay between the number of random measurements, $m$, the prediction error of the final regressor, the training error of the regressor, the ambient dimension $D$, and the intrinsic dimension $d$.

\section{Experimental Results}
\label{sec:expts}
\textbf{Experimental Setup.}
We generated datasets of increasing size, with $D=100$ and $d=30$. These
datasets were generated by first generating a common $U$ matrix of size $D\times d$ with random, orthonormal columns. For a given dataset size, five different $A$ matrices of size $d\times n$ were generated by sampling each entry from a standard normal distribution.
Separate validation and test datasets were also generated by
generating additional random $A$ matrices, in the same way as described before. We generate random $\epsilon_X$ with each entry of matrix $\epsilon_X$ (resp. $\epsilon_Y$) having mean zero and variance $\sigma_x^2$ (resp. $\sigma_y^2$). In order to simulate missing data, we retain each element of each observed feature vector with probability $p$, and in the test and validation datasets with probability $q$. While in the theoretical results, for ease of analysis we assumed that the set $\Omega$, is chosen uniformly at random, with replacement, from the set $\{1,2,\ldots,D\}$, for our practical implementations, we choose $\Omega$ of size $m$ by choosing each feature with probability $p$. Hence, $\bbE m\approx pD$. Previous analyses have shown that these two sampling strategies behave similarly~\citep{recht2011simpler}.  All the results reported here are averaged over the 
five different random datasets that we generated. Similarly, $w$ vector used in our model was generated at
random from a Gaussian distribution, and random coordinates of $w$
were set to 0. The sparsity level of $w$ was set to $10$. $\gamma$ is set to $0.001$. $\lambda_1,\lambda_2,\lambda_3$ are chosen by using a held-out validation set, and searching for parameter values which give the smallest MSE. We found that the performance of SLRM is not very sensitive to the values of $\lambda$'s, and hence a coarse range is enough during validation. $d$ is set by performing PCA on a subset of the data, and calculating how many dimensions are required to capture about $99\%$ of the variance

We compared our algorithm with a stochastic version of MPCR (PCR modified to handle missing data, detailed in Section~\ref{sec:learning_alg}), which uses the PETRELS algorithm to perform Step 1 of MPCR, and then follows it by a stochastic projected gradient method to solve the LASSO problem in Step 2 of MPCR. We shall call this stochastic implementation SMPCR.  

For both SLRM and SMPCR, we allow multiple passes over the dataset in our experiments. The maximum number of passes is fixed to 500 for both SLRM and SMPCR. $\eta_t$ used in Step 8 of SLRM is chosen to be a constant, $\rho$, for a fixed number of rounds, and then allowed to decay as $\rho/t$. This strategy has also been used advocated in~\citep{murata1998statistical,mairal2012task}, and we use this method in our algorithm. The value of $\rho$ was found by trying a range of $\rho$, and choosing the one that gave the best error rate over the hold-out dataset.  

{\bf Experiments in the noiseless setting.} In our first set of experiments, we set $\sigma_x=\sigma_y=0$.  Figure~\ref{fig:syn_no} shows the error bars for the mean squared error (MSE) on the test dataset for both SLRM (in solid, red line), and SMPCR (in broken, blue line). As we can see from the figure, the performance of both SMPCR and SLRM improves with increasing $n$. Figure~\ref{fig:syn_no} also indicates that the average MSE of SLRM is lower than that of SMPCR for all $n$. 
 \begin{figure*}[t]
  \centering
  \subfigure[Noiseless\label{fig:syn_no}]{\includegraphics[scale=0.20]{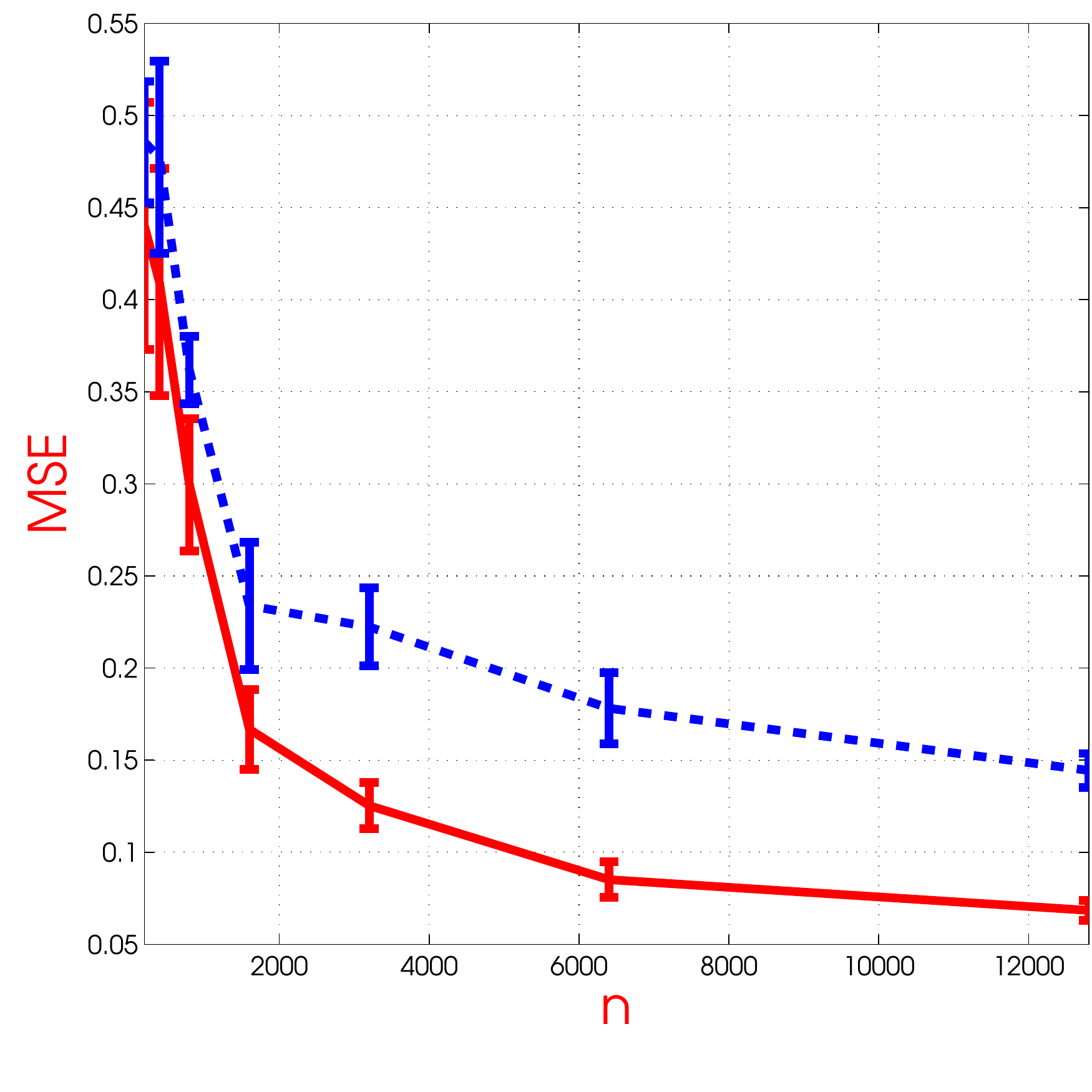}}
  \subfigure[Non-zero $\sigma_x^2; \sigma_y^2 = 0$\label{fig:syn_feat}]{\includegraphics[scale=0.20]{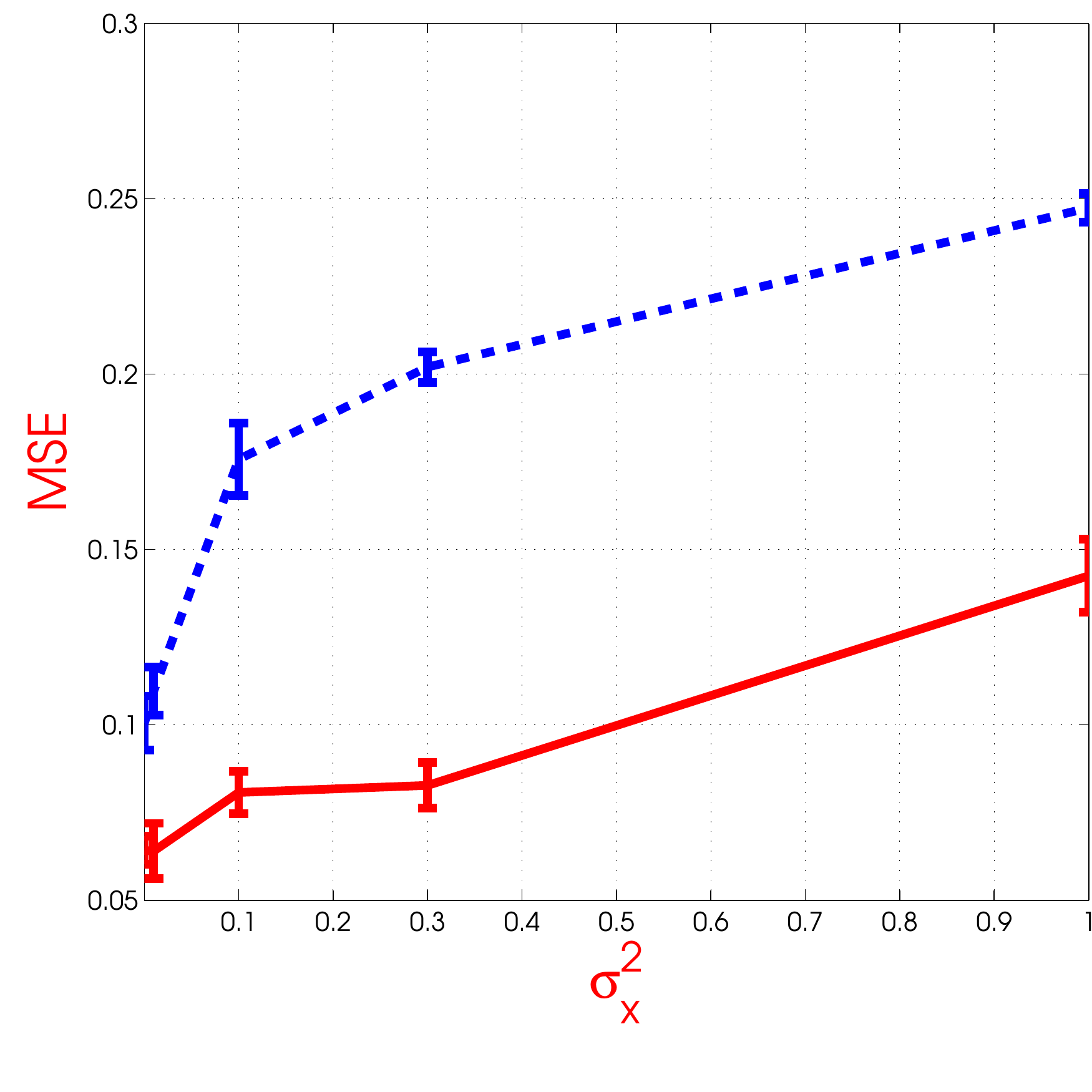}}
 \subfigure[Non-zero $\sigma_x^2,\sigma_y^2$\label{fig:syn_feat_lab}]{\includegraphics[scale=0.20]{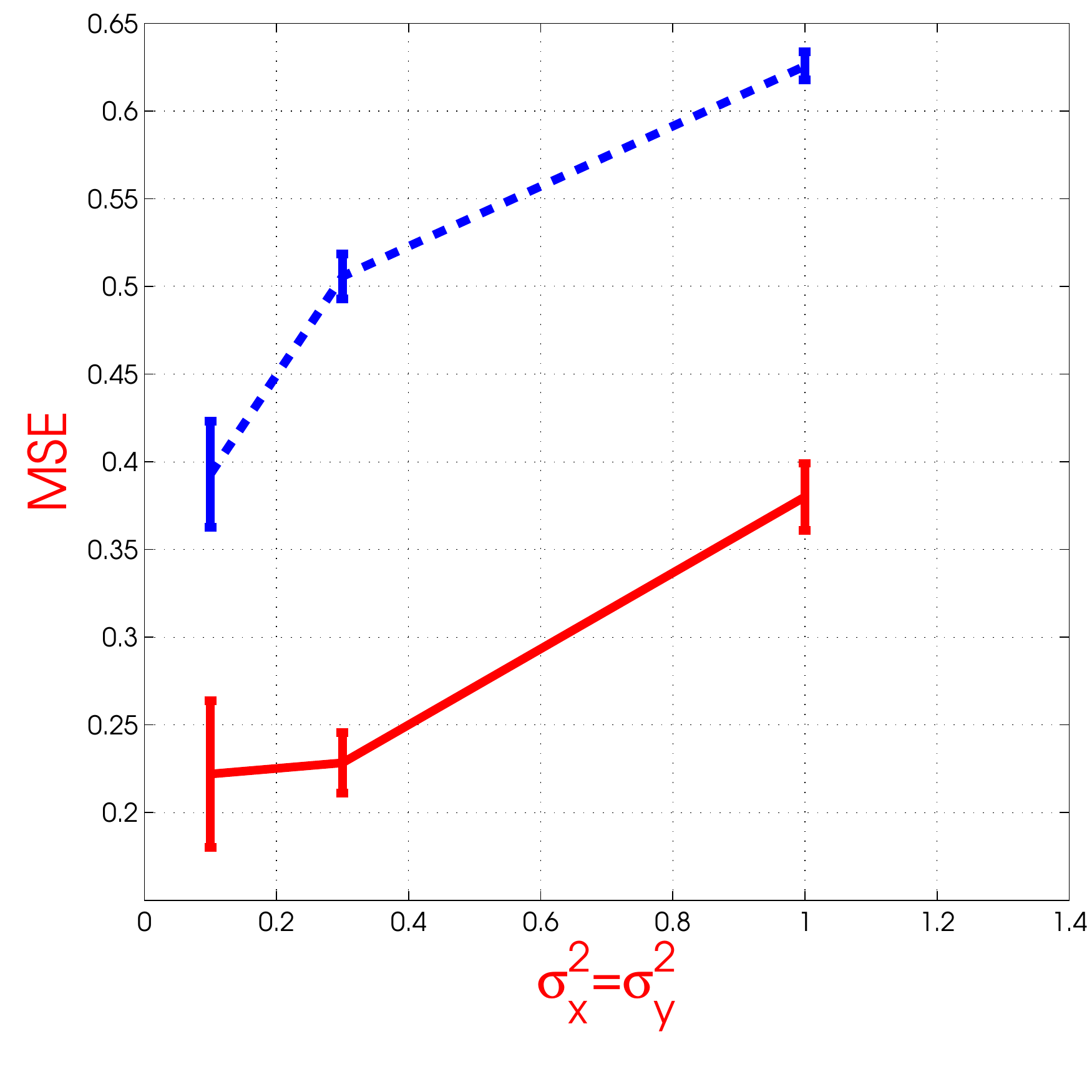}}
 \subfigure[Impact of $p$, probability of feature being observed (not missing)\label{fig:syn_p}]{\includegraphics[scale=0.20]{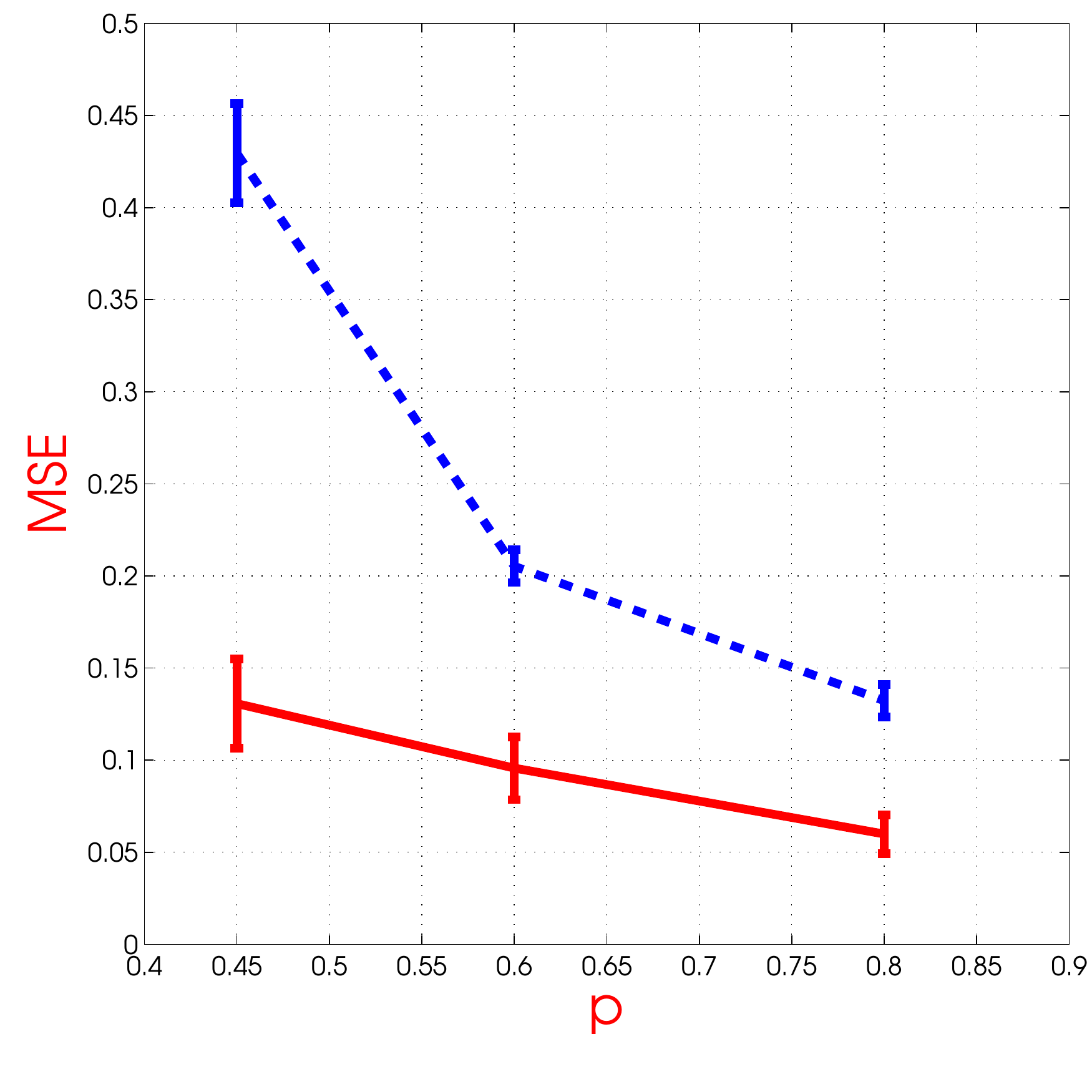}} 
\caption{Comparison between SLRM (solid, red line) and SMPCR (broken, blue line) on synthetic datasets for $D = 100$ and $d = 30$.\label{fig:synthetic}}
\end{figure*}

\begin{figure*}[t]
  \centering
  \subfigure[Real datasets\label{fig:real}]{\includegraphics[scale=0.25]{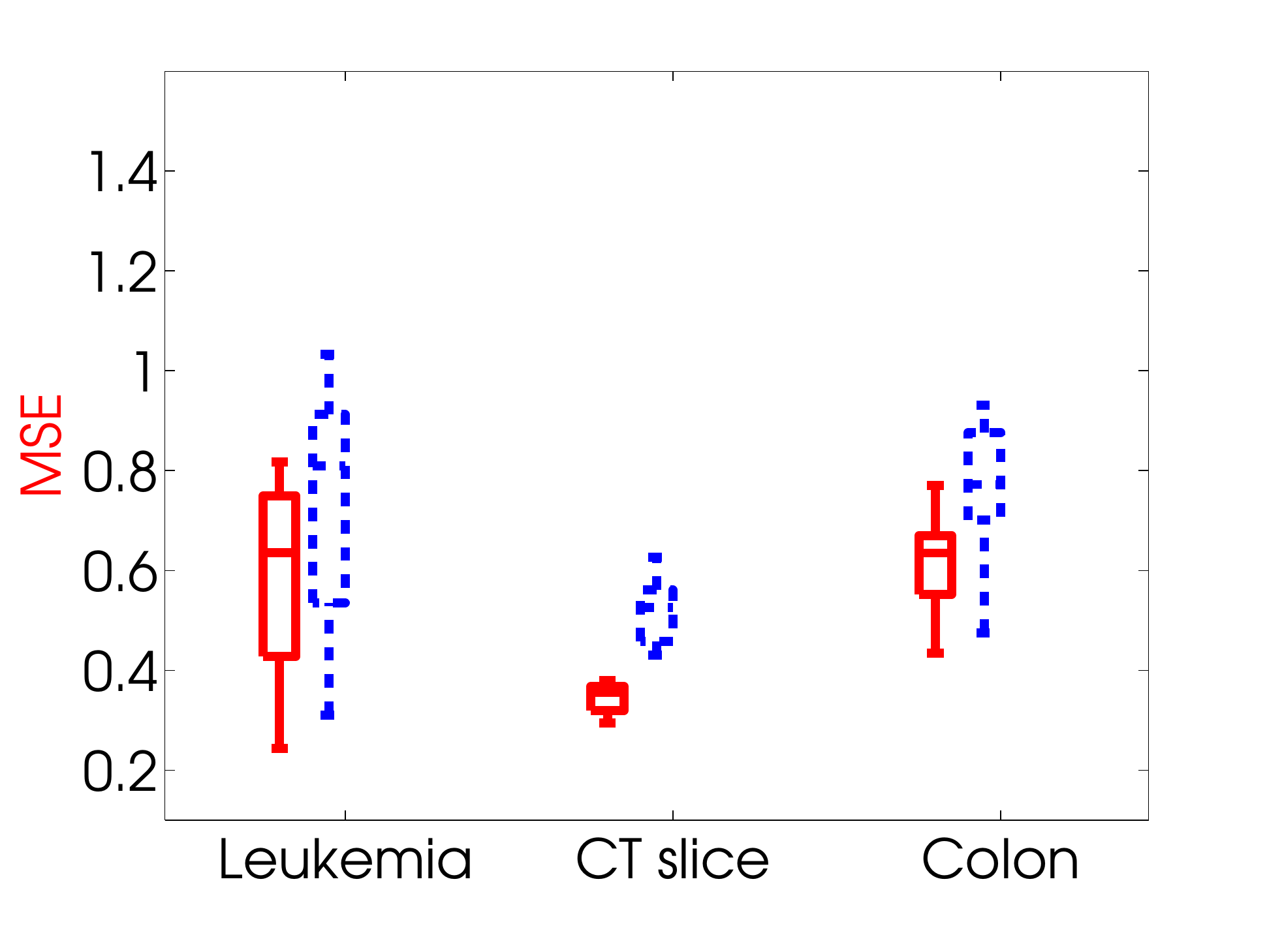}}
  \subfigure[ATP1d\label{fig:atp1d}]{\includegraphics[scale=0.25]{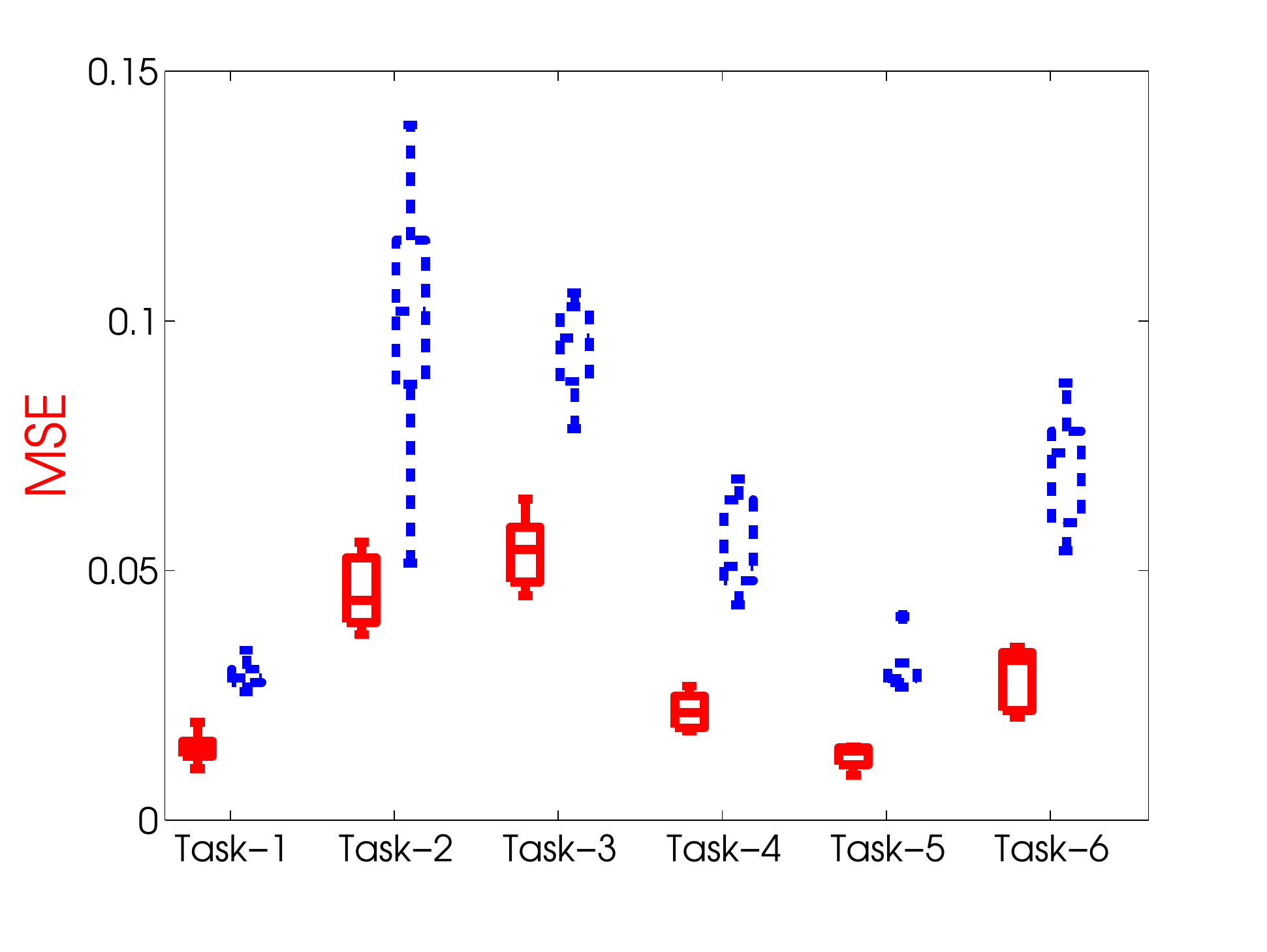}}
 \subfigure[ATP7d\label{fig:atp7d}]{\includegraphics[scale=0.25]{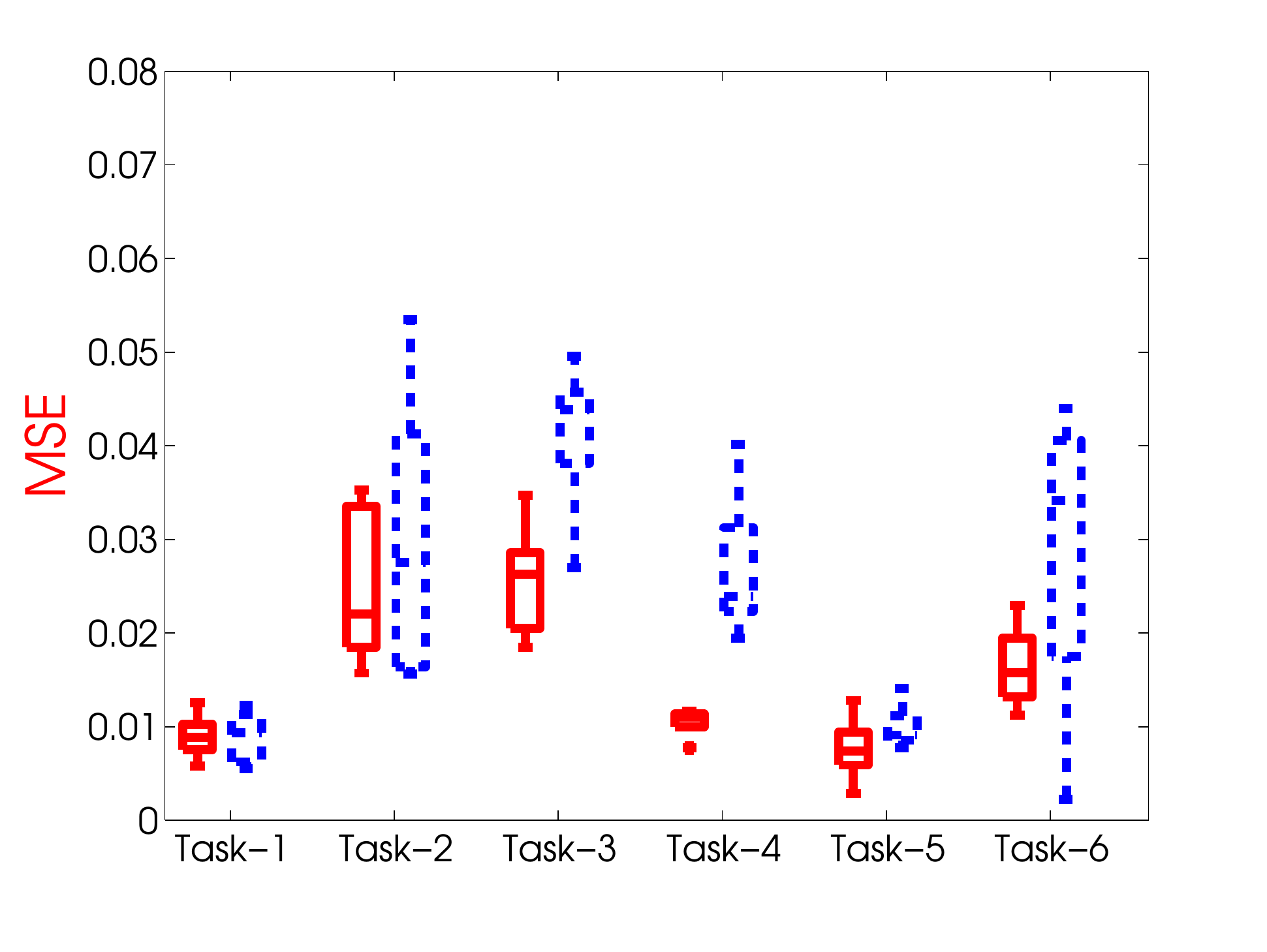}} 
\caption{Comparison between SLRM (solid, red boxes) and SMPCR (broken, blue boxes).\label{fig:real_all}}
\end{figure*}

{\bf Impact of non-zero $\sigma_x^2$.}
While the above experiments, demonstrate the superior performance of SLRM over SMPCR in the noiseless setting, it does not tell us how these algorithms perform in
the presence of noise. We shall now study the impact of non-zero $\sigma_x^2$ on the performance of both SLRM and SMPCR.
For a clearer understanding, we set $\sigma_y^2=0$, and fix the
size of the dataset to $n=12800$. Figure~\ref{fig:syn_feat}
shows the impact of increasing $\sigma_x^2$ on the MSE of
SLRM, and SMPCR.  As we can see from this figure, the MSE of
both SLRM and SMPCR gradually increases with increasing
$\sigma_x^2$. Like in the noiseless setting, the MSE of SLRM is always substantially smaller than the SMPCR method. 

{\bf Impact of non-zero $\sigma_x^2,\sigma_y^2$.}
We shall now examine the effect of non-zero values for
$\sigma_x^2,\sigma_y^2$ on the MSE of SLRM and SMPCR. In these
experiments, we fixed the size of the dataset to $n=12800$, and
increased $\sigma_x^2,\sigma_y^2$. For the sake of simplicity, we keep
$\sigma_x^2=\sigma_y^2$. As we can see
from Figure~\ref{fig:syn_feat_lab}, the MSE of both SLRM and SMPCR
increases with increasing $\sigma_x,\sigma_y$. In this case too, the
MSE of SLRM is substantially smaller than that of SMPCR. From our plots, SLRM seems to be more noise tolerant than SMPCR. 

{\bf Impact of increasing $p$.}
In this experiment, we examine the impact of increasing $p$ on the
error rate of the proposed learning algorithms. We fix
$\sigma_x^2=\sigma_y^2=1$, and the size of dataset to $n=12800$. Like in previous experiments $q$ is set to 0.75. As expected, the error rate of both SLRM, and SMPCR goes down as $p$ increases. 

\subsection{Experimental Results on Real datasets}
We performed experimental comparisons on 15 real world tasks. While an extended discussion of our datasets, has been relegated to the appendix, in this paper, we shall provide a brief description of the datasets. Our datasets are Leukemia, CT Slice, ATP1d, and ATP7d. Both ATP1d, and ATP7d~\citep{groves2011regression} have six tasks each related to airline ticket price prediction. The CT-slice dataset consists of 384 features obtained from CT scan images and the task is to estimate the relative location of the CT slice on the axial axis of human body. The Leukemia dataset~\citep{golub1999molecular}, and the colon-cancer dataset are high dimensional datasets with $D=7129$ and $D=2000$ respectively. Both these datasets have binary labels, $\{-1,+1\}$ as the target but for this paper we treat it as a regression problem. 

From Figures~\ref{fig:real}-~\ref{fig:atp7d}, it is clear that \textbf{the median MSE of SLRM is always lesser than the median MSE of SMPCR on all the tasks}.~\footnote{The horizontal bar in the barplots show the median MSE of the method} When labels are quantized, such as in classification problems, or noisy, the advantage gained from utilizing label information is limited. This explains why on the leukemia and colon cancer datasets, SLRM might, at times perform worse than SMPCR. As mentioned in Section~\ref{sec:learning_alg}, the superior performance of SLRM over SMPCR, on the CT slice dataset can be explained by the fact that for CT slice dataset, both $D=384$ and $d=181$ are large. On both ATP1d and ATP7d datasets, on almost all of the tasks, SLRM far outperforms SMPCR. 

\section{Conclusions and Future Work}
This paper studies the problem of regression with missing data. We proposed a new statistical model and an optimization based approach for learning the parameters of the model. We established risk bounds for our regressor, and demonstrate superior empirical performance over competing algorithms. This work can be extended in several ways. Instead of a single subspace assumption, it should be possible to extend our framework to handle the case when data is generated from a union of $K$ subspaces, using ideas in~\citep{xie2012multiscale}. Our framework can be  extended to handle other tasks using different loss functions and to multi-task learning problems via the use of appropriate matrix norms. 
\appendix
\section{Towards Proof of Theorem~\ref{thm:main}}
In this appendix we shall prove the result of Theorem~\ref{thm:main}.
In order to establish Theorem~\ref{thm:main}, we need the following few important definitions and results which have been taken from~\cite{srebro2010smoothness}
\begin{defn}
 Let $\sigma_{1:n}$ be a collection of Rademacher random variables. The worst case empirical Rademacher complexity of a function class $\cF$ is defined as 
  \begin{equation}
    \cRn(\cF)=\sup_{z_{1:n}} \sup_{f\in\cF} \frac{1}{n} |\sum_{i=1}^n h(z_i)\sigma_i|.
  \end{equation}
  Empiricial Rademacher complexity is defined as 
  \begin{equation}
    \hat{\cR}_n(\cF)=\sup_{f\in\cF} \frac{1}{n} |\sum_{i=1}^n h(z_i)\sigma_i|
  \end{equation}
\end{defn}
\begin{lemma}
  \label{lem:risk_sst}
  Let $l$ be an $H$ smooth non-negative loss, such that $\forall y_1,y_2,y_3, | l(y_1,y_2) - l(y_3,y_2)|\leq b$. Let $L(f)=\bbE_{z,y} l(f(z),y)$, and $\Lhat(f)=\frac{1}{n}\sum_{i=1}^n l(f(z_i),y_i)$. Then for any $\delta >0$, we have, with probability at least $1-\delta$, over a random sample of size $n$, for all $f\in \cF$,
\begin{equation}
  L(f)\leq \Lhat(f) +K\left[\sqrt{\Lhat(f)}\left(\sqrt{H}~\log^{1.5}n \cRn(\cF)+\sqrt{\frac{b\log(1/\delta)}{n}}\right)+H\log^3n~\cRn^2(\cF)+\frac{b\log(1/\delta)}{n}\right]
\end{equation} 
\end{lemma}
The above lemma was proved in~\cite{srebro2010smoothness}, and was used for prediction problems without missing data. We shall use the above result for our problem by using $z=(x,\Omegax)$. This will enable us to provide generalization bounds for our regression problem with missing data.
\begin{lemma}~\cite{srebro2010smoothness}
\label{lem:dudley}
  For any function class $\cF$, containing functions $f:\cZ\rightarrow \bbR$, we have that 
  \begin{equation*}
    \hat{\cR}_n(\cF)=\inf_{\alpha\geq 0} \left[4\alpha+10\int_{\alpha}^{\sqrt{\hat{\bbE}f^2}} \sqrt{\frac{\log \cN_2(\epsilon,\cF,z_{1:n})}{n}} ~\mathrm{d}\epsilon\right].
  \end{equation*}
\end{lemma}
\begin{lemma}
  \label{lem:rad}
\begin{equation*}
  \cRn(\cF_{\gamma})\leq \left(\frac{50m}{n}+\frac{14}{\sqrt{n}}\right)\frac{3DR_1B_X}{(1-\gamma)}. 
\end{equation*}
\end{lemma}
\begin{proof}
For the sake of convenience, let $B\defeq
\frac{DR_1}{m(1-\gamma)}$. Also throught this document, for the sake
of conciseness, we shall use
the pair $(x,\Omega)$ to mean $(P_{\Omega}(x),\Omega)$. Similarily
$(x_i,\Omega_i)$ would mean $(P_{\Omega_i}(x_i),\Omega_i)$
Instead of bounding the Rademacher complexity of $\cF$, we shall work with a slightly different function class $\cF_g\defeq \{f(x,\Omega)=\beta^T\xomega :||\beta||_2 \leq B \}$. It is now, enough to control the Rademacher complexity of the function class $\cF_g$, since all functions $f\in\cF_{\gamma}$, can be written as $f(x,\Omega)=\mu^T\xomega$, for an appropriate $\mu$ where, by definition of $\cF_{\gamma}$, and using Cauchy-Schwartz inequality, we can guarantee that $||\mu||_2$ is upper bounded by $B$, and hence $\cRn(\cF_{\gamma})\leq \cRn(\cFg)$. The rest of the proof upper bounds $\cRn(\cFg)$. We control this Rademacher complexity via lemma~\ref{lem:dudley}. Let $f_1,f_2\in\cF_g$. 
\begin{align}
 d_2(f_1,f_2,x_{1:n},\Omega_{1:n})&= \sqrt{\frac{1}{n}\sum_{i=1}^n (f_1(x_i,\Omega_i)-f_2(x_i,\Omega_i))^2}\\
 &= \sqrt{\frac{1}{n}\sum_{i=1}^n (\beta_1^Tx_{i,\Omega_i}-\beta_2^Tx_{i,\Omega_i})^2}\\
 &\leq \sqrt{\frac{1}{n}\sum_{i=1}^n ||\beta_1-\beta_2||_2^2 ||x_{i,\Omega_i}||_2^2}\label{eqn:last_but_one}\\
 &\leq\sqrt{m||X||_{\infty}^2 ||\beta_1-\beta_2||^2 }\label{eqn:last}
\end{align}
where, in order to obtain Equation~\ref{eqn:last} from ~\ref{eqn:last_but_one} we used the fact that  $||x_{i,\Omega_i}||_2^2\leq \sqrt{m ||X||_{\infty}^2}$ 
To upper bound the above R.H.S. by $\epsilon$, we need
$||\beta_1-\beta_2||\leq
\frac{\epsilon}{\sqrt{m}||X||_{\infty}}$. Hence, to cover $\cF_g$, it
is enough to cover an $\ell_2$ ball of radius $B$, with $\ell_2$ ball
of radius $\frac{\epsilon}{\sqrt{m}||X||_{\infty}}$. Now, we know
that the cover a ball of radius $R$, with balls of radius $\epsilon$,
in $d$ dimensions we need $(\frac{3R}{\epsilon})^d$. Using this
result, we can conclude that $\cN_2(\cF_g,z_{1:n})\leq\left(\frac{3B\sqrt{m}{||X||_{\infty}}}{\epsilon}\right)^m$. Plugging, this into the Dudley entropy integral, and using lemma~\ref{lem:dudley}, we get
\begin{equation}
  \chatRn(\cFg)\leq\min_{\alpha\geq 0} 4\alpha+10\int_{\alpha}^{\sup_{f\in \cF_g}\sqrt{\hat{E}(f^2)}} \sqrt{\frac{m}{n}\log\left(\frac{3DR_1||X||_{\infty}}{\epsilon}\right)}~\mathrm{d}\epsilon
\end{equation}
For the sake of simplicity, let us denote by $F\defeq\sup_{f\in \cF_g}\sqrt{\hat{E}(f^2)}$, and by $C\defeq \frac{3DR_1||X||_{\infty}}{\sqrt{m}(1-\gamma)}$. It is easy to see that $F\leq \frac{DR_1||X||_2}{m(1-\gamma)}\leq C$. With this notation, the above inequality can be manipulated as follows
\begin{align}
  \chatRn(\cFg)&\leq 4\alpha+\frac{10}{\sqrt{n}} \int_{\alpha}^F \sqrt{m(\log(C)-\log(\epsilon))}~\mathrm{d}\epsilon\\
  &\leq 4\alpha -\frac{20C\sqrt{m}}{\sqrt{n}}\int_{\sqrt{\log(C)-\log(\alpha)}}^{\sqrt{\log(C)-\log(F)}}\theta^2\exp(-\theta^2)~\mathrm{d}\epsilon 
\end{align}
where the above expression is obtained by the change of variable, $\theta^2=\log(C)-\log(\epsilon)$. Substituting $K_2=\sqrt{\log(C)-\log(F)}$, for the upper limits of the integral appearing above, we get 
\begin{align}
  \chatRn(\cFg)&\leq 4\alpha-\frac{20C\sqrt{m}}{\sqrt{n}}\int_{\sqrt{\log(\frac{C}{\alpha})}}^{K_2}\theta^2\exp(-\theta^2)~\mathrm{d}\theta\\
  &\leq 4\alpha-\frac{10\alpha\sqrt{m}}{\sqrt{n}}\sqrt{\log\left(\frac{C}{\alpha}\right)}+5C\sqrt{\frac{m\pi}{n}}\erf\left(\sqrt{\log\left(\frac{C}{\alpha}\right)}\right)+10C\sqrt{\frac{m}{n}} K_2e^{-K_2^2}-5C\sqrt{\frac{m\pi}{n}}\erf(K_2)\nonumber\\
&\leq 4\alpha +10C\sqrt{\frac{m}{2en}}+5C\sqrt{\frac{m\pi}{n}}\erf\left(\sqrt{\log\left(\frac{C}{\alpha}\right)}\right)\label{eqn:eqn1}
\end{align}
where last equation was obtained by using the inequality $x e^{-x^2}\leq \frac{1}{\sqrt{2e}}$, and by dropping all the negative terms. We can now optimize over $\alpha$, by setting the gradient to 0, to get 
\begin{equation}
  \alpha^{*}=C\exp\left(-\frac{4n+25m+\sqrt{16n^2+200mn}}{50m}\right)
\end{equation} 
Substituting $\alpha^{*}$ for $\alpha$ in Equation~\ref{eqn:eqn1}, and over-estimating $\erf(\cdot)$ by 1, we get 
\begin{equation}
  \chatRn(\cFg)\leq 10C\sqrt{\frac{m}{2en}}+5C\sqrt{\frac{m\pi}{n}}+\frac{4C}{\sqrt{e}} e^{-\frac{0.16n}{m}}\\
\end{equation}
Replacing $C$, by its definition, we get 
\begin{equation}
  \chatRn(\cFg)\leq \left(\frac{10}{\sqrt{2en}}+\frac{5\sqrt{\pi}}{\sqrt{n}}+4e^{-\frac{0.16n}{m}}\right)\frac{3DR_1||X||_{\infty}}{1-\gamma}\leq \frac{45DR_1||X||_{\infty}}{(1-\gamma)\sqrt{n}}
\end{equation}
Since the above quantity is independent of the sample, hence the above bound on $\cRn(\cFg)$ also holds for $\cR(\cFg)$, i.e.
\begin{align}
  \cRn(\cFg)\leq \left(\frac{10}{\sqrt{2en}}+\frac{5\sqrt{\pi}}{\sqrt{n}}+4e^{-\frac{0.16n}{m}}\right)\frac{3DR_1||X||_{\infty}}{1-\gamma}
\leq \left(\frac{50m}{n}+\frac{14}{\sqrt{n}}\right)\frac{3DR_1B_X}{(1-\gamma)}.
\end{align}
\end{proof}
\textbf{Proof of Theorem~\ref{thm:main}}.
The theorem now follows immediately from Lemma~\ref{lem:risk_sst},  the squared loss, which is by definition 2-smooth and
Lemma~\ref{lem:rad}. The role of $z$ is played by the random pair
$(x,\Omega)$. The following trivial bound for $b$, required in
Theorem~\ref{lem:risk_sst} holds\begin{equation}
  b\leq |l(y_1,y_2)-l(y_1,y_3)|\leq   |l(y_1,y_2)|+|l(y_1,y_3)|\leq  2\left(B_Y+\frac{DR_1}{m(1-\gamma)}\right)^2
\end{equation}
\section{Description of datasets}
In this appendix we provide information regarding the datasets that were used in our experiments in Section~\refeq{sec:expts}.
We shall provide a description of the datasets that were used in Section 6.
\begin{enumerate}
\item \textbf{Leukamia dataset.} The Leukamia dataset was obtained from  \url{http://www.csie.ntu.edu.tw/~cjlin/libsvmtools/datasets/}. It is a cancer classification dataset~\cite{golub1999molecular} where the features are gene expression levels, and our task is to classify whether the gene expression levels indicate acute myeloid leukemia or acute lymphoblastic leukemia. This dataset comes with  separate train and test datasets of a total of 72 points with $D=7129$. For our experiments we merge both the train and test datasets, and randomly sample 30 data points for training, 22 data points for testing, and the remaining for validation. This was repeated five times to obtain five different training, test and validation datasets. On each of the training sets, we retained a feature with probability 0.1, to simulate the missing data scenario. 
\item \textbf{CT slice.} The CT slice dataset was obtained from~\url{https://archive.ics.uci.edu/ml/datasets/Relative+location+of+CT+slices+on+axial+axis}. The dataset consists of 384 features obtained from CT scan images. The task is to estimate the relative location of CT slice on the axial axis of human body. The original dataset consisted of 53500 data points. For our experiments we sampled 300 data points uniformly at random, of which 100 each were used for training, testing, and validation. This process was repeated five times, to obtain five different datasets. 
\item \textbf{ATP1d, ATP7d.} The ATP1d, ATP7d datasets were obtained from~\url{http://mulan.sourceforge.net/datasets-mtr.html}. The ATP1d and ATP7d datasets are airline ticket price prediction datasets. where the problem is to predict the prices for six target flight preferences, namely the price of any non-stop flight, Delta airlines, Continental airlines, Airtran, and United airlines. While for the ATP1d dataset these target prices are the next day price, for ATP7d the targets are the minimum price observed over the next 7 days.  The input features for each
sample are values that are useful for prediction of the airline ticket prices
for a specific observation date-departure date pair. The features include quantities like day-of-the-week of the observation date, number of days between observation date and departure, and several other price related features such as minimum quoted price, mean quoted price etc...In order to normalize our features, we divided all our price related features by 1000, and the feature which measures the number of days between observation date and the day of departure by 180. See~\cite{spyromitros2012multilabel,groves2011regression} for more details on these datasets. We converted the cardinal day-of-the-week feature into a 7 dimensional boolean vector. The sizes of our training, testing, and validation datasets are 200,46,50 respectively. These were obtained via random sampling from the original dataset. 
\end{enumerate}
\bibliographystyle{natbib}
\footnotesize{\bibliography{slurm_arxiv}}

\begin{thebibliography}{}

\bibitem[Abed-Meraim {\em et~al.}(2000)Abed-Meraim, Chkeif, and
  Hua]{abed2000fast}
Abed-Meraim, K., Chkeif, A., and Hua, Y. (2000).
\newblock Fast orthonormal past algorithm.
\newblock {\em Signal Processing Letters, IEEE\/}, {\bf 7}(3), 60--62.

\bibitem[Balzano {\em et~al.}(2010a)Balzano, Recht, and Nowak]{balzano2010high}
Balzano, L., Recht, B., and Nowak, R. (2010a).
\newblock High-dimensional matched subspace detection when data are missing.
\newblock In {\em Information Theory Proceedings (ISIT), 2010 IEEE
  International Symposium on\/}, pages 1638--1642. IEEE.

\bibitem[Balzano {\em et~al.}(2010b)Balzano, Nowak, and
  Recht]{balzano2010online}
Balzano, L., Nowak, R., and Recht, B. (2010b).
\newblock Online identification and tracking of subspaces from highly
  incomplete information.
\newblock In {\em Communication, Control, and Computing (Allerton), 2010 48th
  Annual Allerton Conference on\/}, pages 704--711. IEEE.

\bibitem[Cand{\`e}s and Recht(2009)Cand{\`e}s and Recht]{candes2009exact}
Cand{\`e}s, E.~J. and Recht, B. (2009).
\newblock Exact matrix completion via convex optimization.
\newblock {\em Foundations of Computational mathematics\/}, {\bf 9}(6),
  717--772.

\bibitem[Chi {\em et~al.}(2012)Chi, Eldar, and Calderbank]{chi2012petrels}
Chi, Y., Eldar, Y.~C., and Calderbank, R. (2012).
\newblock Petrels: Subspace estimation and tracking from partial observations.
\newblock In {\em Acoustics, Speech and Signal Processing (ICASSP), 2012 IEEE
  International Conference on\/}, pages 3301--3304. IEEE.

\bibitem[Fukumizu {\em et~al.}(2009)Fukumizu, Bach, Jordan, {\em
  et~al.}]{fukumizu2009kernel}
Fukumizu, K., Bach, F.~R., Jordan, M.~I., {\em et~al.} (2009).
\newblock Kernel dimension reduction in regression.
\newblock {\em The Annals of Statistics\/}, {\bf 37}(4), 1871--1905.

\bibitem[Goldberg {\em et~al.}(2010)Goldberg, Recht, Xu, Nowak, and
  Zhu]{goldberg2010transduction}
Goldberg, A., Recht, B., Xu, J., Nowak, R., and Zhu, X. (2010).
\newblock Transduction with matrix completion: Three birds with one stone.
\newblock In {\em Advances in neural information processing systems\/}, pages
  757--765.

\bibitem[Golub {\em et~al.}(1999)Golub, Slonim, Tamayo, Huard, Gaasenbeek,
  Mesirov, Coller, Loh, Downing, Caligiuri, {\em et~al.}]{golub1999molecular}
Golub, T.~R., Slonim, D.~K., Tamayo, P., Huard, C., Gaasenbeek, M., Mesirov,
  J.~P., Coller, H., Loh, M.~L., Downing, J.~R., Caligiuri, M.~A., {\em et~al.}
  (1999).
\newblock Molecular classification of cancer: class discovery and class
  prediction by gene expression monitoring.
\newblock {\em science\/}, {\bf 286}(5439), 531--537.

\bibitem[Groves and Gini(2011)Groves and Gini]{groves2011regression}
Groves, W. and Gini, M. (2011).
\newblock A regression model for predicting optimal purchase timing for airline
  tickets.
\newblock Technical report, Technical Report 11-025, University of Minnesota,
  Minneapolis, MN.

\bibitem[Hardt(2013)Hardt]{hardt2013understanding}
Hardt, M. (2013).
\newblock Understanding alternating minimization for matrix completion.
\newblock {\em arXiv preprint arXiv:1312.0925\/}.

\bibitem[Hastie {\em et~al.}(2003)Hastie, Tibshirani, and
  Friedman]{hastie2003esl}
Hastie, T., Tibshirani, R., and Friedman, J.~H. (2003).
\newblock {\em The Elements of Statistical Learning\/}.
\newblock Springer.

\bibitem[Hua {\em et~al.}(1999)Hua, Xiang, Chen, Abed-Meraim, and
  Miao]{hua1999new}
Hua, Y., Xiang, Y., Chen, T., Abed-Meraim, K., and Miao, Y. (1999).
\newblock A new look at the power method for fast subspace tracking.
\newblock {\em Digital Signal Processing\/}, {\bf 9}(4), 297--314.

\bibitem[Jain {\em et~al.}(2013)Jain, Netrapalli, and Sanghavi]{jain2013low}
Jain, P., Netrapalli, P., and Sanghavi, S. (2013).
\newblock Low-rank matrix completion using alternating minimization.
\newblock In {\em Proceedings of the 45th annual ACM symposium on Symposium on
  theory of computing\/}, pages 665--674. ACM.

\bibitem[Koren {\em et~al.}(2009)Koren, Bell, and Volinsky]{koren2009matrix}
Koren, Y., Bell, R., and Volinsky, C. (2009).
\newblock Matrix factorization techniques for recommender systems.
\newblock {\em Computer\/}, {\bf 42}(8), 30--37.

\bibitem[Krim {\em et~al.}(1995)Krim, Viberg, {\em et~al.}]{krim1995sensor}
Krim, H., Viberg, M., {\em et~al.} (1995).
\newblock Sensor array signal processing: two decades later.

\bibitem[Loh and Wainwright(2011)Loh and Wainwright]{loh2011high}
Loh, P.-L. and Wainwright, M.~J. (2011).
\newblock High-dimensional regression with noisy and missing data: Provable
  guarantees with non-convexity.
\newblock In {\em Advances in Neural Information Processing Systems\/}, pages
  2726--2734.

\bibitem[Mairal {\em et~al.}(2009)Mairal, Bach, Ponce, and
  Sapiro]{mairal2009online}
Mairal, J., Bach, F., Ponce, J., and Sapiro, G. (2009).
\newblock Online dictionary learning for sparse coding.
\newblock In {\em Proceedings of the 26th Annual International Conference on
  Machine Learning\/}, pages 689--696. ACM.

\bibitem[Mairal {\em et~al.}(2012)Mairal, Bach, and Ponce]{mairal2012task}
Mairal, J., Bach, F., and Ponce, J. (2012).
\newblock Task-driven dictionary learning.
\newblock {\em Pattern Analysis and Machine Intelligence, IEEE Transactions
  on\/}, {\bf 34}(4), 791--804.

\bibitem[Maurer and Pontil(2010)Maurer and Pontil]{maurer2010k}
Maurer, A. and Pontil, M. (2010).
\newblock K-dimensional coding schemes in hilbert spaces.
\newblock {\em Information Theory, IEEE Transactions on\/}, {\bf 56}(11),
  5839--5846.

\bibitem[Murata(1998)Murata]{murata1998statistical}
Murata, N. (1998).
\newblock A statistical study of on-line learning.
\newblock {\em Online Learning and Neural Networks. Cambridge University Press,
  Cambridge, UK\/}.

\bibitem[Oja(1982)Oja]{oja1982simplified}
Oja, E. (1982).
\newblock Simplified neuron model as a principal component analyzer.
\newblock {\em Journal of mathematical biology\/}, {\bf 15}(3), 267--273.

\bibitem[Recht(2011)Recht]{recht2011simpler}
Recht, B. (2011).
\newblock A simpler approach to matrix completion.
\newblock {\em The Journal of Machine Learning Research\/}, {\bf 12},
  3413--3430.

\bibitem[Roy and Kailath(1989)Roy and Kailath]{roy1989esprit}
Roy, R. and Kailath, T. (1989).
\newblock Esprit-estimation of signal parameters via rotational invariance
  techniques.
\newblock {\em Acoustics, Speech and Signal Processing, IEEE Transactions
  on\/}, {\bf 37}(7), 984--995.

\bibitem[Spyromitros-Xioufis {\em et~al.}(2012)Spyromitros-Xioufis, Tsoumakas,
  Groves, and Vlahavas]{spyromitros2012multilabel}
Spyromitros-Xioufis, E., Tsoumakas, G., Groves, W., and Vlahavas, I. (2012).
\newblock Multi-label classification methods for multi-target regression.

\bibitem[Srebro {\em et~al.}(2010)Srebro, Sridharan, and
  Tewari]{srebro2010smoothness}
Srebro, N., Sridharan, K., and Tewari, A. (2010).
\newblock Smoothness, low noise and fast rates.
\newblock In {\em Advances in Neural Information Processing Systems\/}, pages
  2199--2207.

\bibitem[Suzuki and Sugiyama(2013)Suzuki and Sugiyama]{suzuki2013sufficient}
Suzuki, T. and Sugiyama, M. (2013).
\newblock Sufficient dimension reduction via squared-loss mutual information
  estimation.
\newblock {\em Neural computation\/}, {\bf 25}(3), 725--758.

\bibitem[Szlam and Sapiro(2009)Szlam and Sapiro]{szlam2009discriminative}
Szlam, A. and Sapiro, G. (2009).
\newblock Discriminative k-metrics.
\newblock In {\em Proceedings of the 26th Annual International Conference on
  Machine Learning\/}, pages 1009--1016. ACM.

\bibitem[Tuncer and Friedlander(2009)Tuncer and
  Friedlander]{tuncer2009classical}
Tuncer, T.~E. and Friedlander, B. (2009).
\newblock {\em Classical and modern direction-of-arrival estimation\/}.
\newblock Academic Press.

\bibitem[Vainsencher {\em et~al.}(2011)Vainsencher, Mannor, and
  Bruckstein]{vainsencher2011sample}
Vainsencher, D., Mannor, S., and Bruckstein, A.~M. (2011).
\newblock The sample complexity of dictionary learning.
\newblock {\em The Journal of Machine Learning Research\/}, {\bf 12},
  3259--3281.

\bibitem[Xie {\em et~al.}(2012)Xie, Huang, and Willett]{xie2012multiscale}
Xie, Y., Huang, J., and Willett, R. (2012).
\newblock Multiscale online tracking of manifolds.
\newblock In {\em Statistical Signal Processing Workshop (SSP), 2012 IEEE\/},
  pages 620--623. IEEE.

\bibitem[Yang(1995)Yang]{yang1995projection}
Yang, B. (1995).
\newblock Projection approximation subspace tracking.
\newblock {\em Signal Processing, IEEE Transactions on\/}, {\bf 43}(1),
  95--107.

\end{thebibliography}
\end{document}